\documentclass{article} 
\usepackage{iclr2023_conference,times}

\usepackage{amsmath,amsfonts,bm}

\def\eqref#1{equation~\ref{#1}}

\def\1{\bm{1}}

\DeclareMathAlphabet{\mathsfit}{\encodingdefault}{\sfdefault}{m}{sl}
\SetMathAlphabet{\mathsfit}{bold}{\encodingdefault}{\sfdefault}{bx}{n}

\usepackage{hyperref}
\usepackage{url}
\usepackage{resizegather}
\usepackage{graphicx}
\usepackage{subfigure}
\usepackage{capt-of}
\usepackage{caption}
\usepackage{threeparttable}

\usepackage[ruled]{algorithm2e}

\usepackage{xcolor}

\usepackage{booktabs}       
\usepackage{wrapfig}       
\usepackage{diagbox}
\usepackage{multirow}
\usepackage{makecell}

\usepackage{bm}
\usepackage{dsfont}
\usepackage{amsmath}
\usepackage{amssymb}
\usepackage{amsthm}
\newtheorem{theorem}{Theorem}

\def \cE {\mathcal{E}}

\def \P {\mathbb{P}}

\def \bfE {\mathbb{E}}
 \def \cU {\mathcal{U}}
 \def \cI {\mathcal{I}}
\def \cL{\mathcal{L}}
\def \cO{\mathcal{O}}
 \def \cD{\mathcal{D}}

 \def \cV {\mathcal{V}}

\title{StableDR: Stabilized Doubly Robust Learning for Recommendation on Data Missing Not at Random}

\author{Haoxuan Li$^1$\qquad Chunyuan Zheng$^2$\qquad Peng Wu$^3$\thanks{Corresponding author.}\\
    $^1$Peking University \\
    $^2$University of California, San Diego \\
    $^3$Beijing Technology and Business University\\
     \texttt{hxli@stu.pku.edu.cn},\qquad \texttt{czheng@ucsd.edu},\qquad  \texttt{pengwu@btbu.edu.cn}
}

\iclrfinalcopy 
\begin{document}

\maketitle

\begin{abstract}
 In recommender systems, users always choose the favorite items to rate, which leads to data missing not at random and poses a great challenge for unbiased evaluation and learning of prediction models. Currently, the doubly robust (DR) methods have been widely studied and demonstrate superior performance. However, in this paper, we show that DR methods are unstable and have unbounded bias, variance, and generalization bounds to extremely small propensities. Moreover, the fact that DR relies more on extrapolation will lead to suboptimal performance. To address the above limitations while retaining double robustness, we propose a stabilized doubly robust (StableDR) learning approach with a weaker reliance on extrapolation. Theoretical analysis shows that StableDR has bounded bias, variance, and generalization error bound simultaneously under inaccurate imputed errors and arbitrarily small propensities. In addition, we propose a novel learning approach for StableDR that updates the imputation, propensity, and prediction models cyclically, achieving more stable and accurate predictions. Extensive experiments show that our approaches significantly outperform the existing methods.
\end{abstract}

\section{Introduction}

Modern recommender systems (RSs) are rapidly evolving with the adoption of sophisticated deep learning models~\citep{zhang2019deep}. However, 
it is well documented that directly using advanced deep models 
usually achieves sub-optimal performance 
due to the existence of various biases in RS~\citep{Chen-etal2020, Wu-etal2022}, and the biases would be amplified over time~\citep{Mansoury-etal2020, Wen-etal2022}. 
 A large number of debiasing methods have emerged and gradually become a trend. 
 For many practical tasks in RS, such as 
rating prediction~\citep{Schnabel-Swaminathan2016, wang-etal2020, Wang-Zhang-Sun-Qi2019}, post-view click-through rate  prediction~\citep{MRDR}, post-click conversion rate prediction~\citep{Zhang-etal2020, Dai-etal2022}, and uplift modeling~\citep{SDM-SaitoSN19, Sato-Singh2019, Sato-Takemori2020}, a critical challenge is to combat the selection bias and confounding bias that leading to significantly difference between the trained sample and the  targeted population~\citep{Hernan-Robins2020}. Various methods were designed to
  address this problem and among them,  
  doubly robust (DR) methods~\citep{Wang-Zhang-Sun-Qi2019, Zhang-etal2020, Chen-etal2021, Dai-etal2022, Ding-etal2022} play the dominant role due to their better performance and theoretical properties.

The success of DR is attributed to its double robustness and joint-learning technique. However, the DR methods still have many limitations. Theoretical analysis shows that {inverse probability scoring (IPS)} and DR methods may have infinite bias, variance, and generalization error bounds, in the presence of extremely small propensity scores~\citep{Schnabel-Swaminathan2016, Wang-Zhang-Sun-Qi2019, MRDR, TDR}. In addition, due to the fact that users are more inclined to evaluate the preferred items, the problem of data missing not at random (MNAR) often occurs in RS. This would cause selection bias and results in inaccuracy for methods that more rely on extrapolation, such as error imputation based (EIB)~\citep{marlin2012collaborative, steck2013evaluation} and DR methods.


To overcome the above limitations while maintaining double robustness, we propose a stabilized doubly robust (SDR) estimator with a weaker reliance on extrapolation, which reduces the negative impact of extrapolation and MNAR effect on the imputation model. Through theoretical analysis, we demonstrate that the SDR has bounded bias and generalization error bound for arbitrarily small propensities, which further indicates that the SDR can achieve more stable predictions.

Furthermore, we propose a novel cycle learning approach for SDR. Figure \ref{fig:sdr} shows the differences between the proposed cycle learning of SDR and  the existing unbiased learning approaches. Two-phase learning~\citep{marlin2012collaborative, steck2013evaluation, Schnabel-Swaminathan2016} first obtains an imputation/propensity model to estimate the ideal loss and then updates the prediction model by minimizing the estimated loss. DR-JL~\citep{Wang-Zhang-Sun-Qi2019}, MRDR-DL~\citep{MRDR}, and AutoDebias~\citep{Chen-etal2021} alternatively update the model used to estimate the ideal loss and the prediction model. The proposed learning method cyclically uses different losses to update the three models with the aim of achieving more stable and accurate prediction results. We have conducted extensive experiments on two real-world datasets, and the results show that the proposed approach significantly improves debiasing and convergence performance compared to the existing methods.

\begin{figure}[tbp]
    \centering
    \resizebox{1\linewidth}{!}{
    \includegraphics[scale=1]{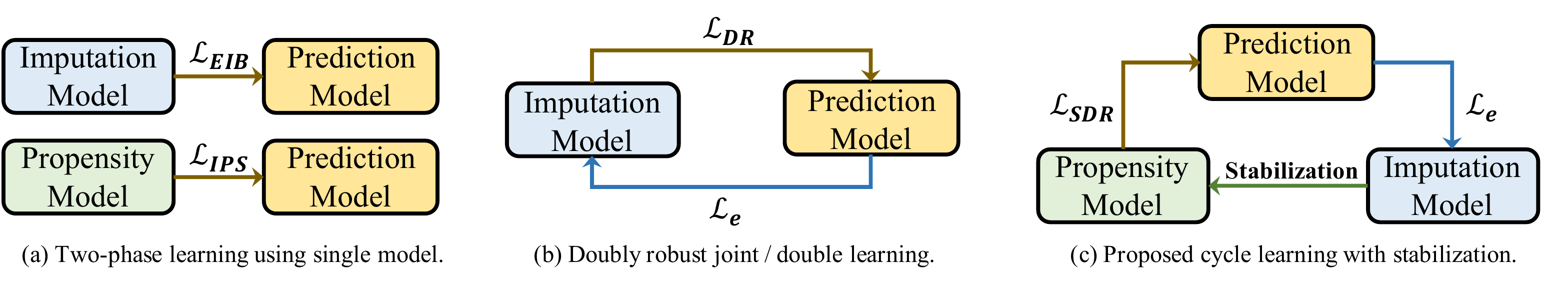}}
     \caption{During the training of updating a prediction model, two-phase learning~\citep{marlin2012collaborative, steck2013evaluation, Schnabel-Swaminathan2016} uses a fixed imputation/propensity model (Left), whereas DR-JL~\citep{Wang-Zhang-Sun-Qi2019}, MRDR-DL~\citep{MRDR}, and AutoDebias~\citep{Chen-etal2021} uses alternative learning between the imputation/propensity and the prediction model (Middle). The proposed learning approach updates the three models cyclically with stabilization (Right).}
    \label{fig:sdr} 
    \vskip -0.1in
\end{figure}

\section{Preliminaries} 
\subsection{Problem Setting}


In RS, due to the fact that users are more inclined to evaluate the preferred items, the collected ratings are always missing not at random (MNAR). We formulate the 
data MNAR problem 
using the widely adopted
potential outcome framework~\citep{Neyman1990, Imbens-Rubin2015}. Let $\cU = \{1, 2, ..., U\}$, $\cI = \{1, 2, ..., I\}$ and $\cD = \cU \times \cI$
  be the index sets of users, items, 
  all user-item pairs.  
 For each $(u,i) \in \cD$, we have a treatment    $o_{u,i} \in \{0, 1\}$, a feature vector $x_{u,i}$, and an observed rating $r_{u,i}$, 
where $o_{u,i}=1$ if user $u$ rated the item $i$ in the logging data, $o_{u,i} = 0$ if the rating is missing.
Let $r_{u,i}(1)$  
 is defined as the  
be the rating that would be observed if item $i$ had been rated by user $u$, which is observable only for $\cO = \{ (u,i) \mid (u,i)\in \cD, o_{u,i} =  1 \}$. Many tasks in RS can be formulated by
predicting the potential outcome  $r_{u,i}(1)$ using feature $x_{u,i}$ for each $(u,i)$. 

Let $\hat r_{u,i}(1) = f(x_{u,i}; \phi)$ be a prediction model with parameters $\phi$. If all the potential outcomes $\{ r_{u,i}(1): (u,i) \in \cD \}$ were observed, the ideal loss function for solving parameters $\phi$ is given as 
		\begin{equation*}
				\cL_{ideal}(\phi)  
				=  |\cD|^{-1}  \sum_{(u,i) \in \cD} e_{u,i},  
		\end{equation*}    
where $e_{u,i}$ is the prediction error, such as the squared loss  
$e_{u,i} = (\hat r_{u,i}(1)  - r_{u,i}(1) )^2$.  
$\cL_{ideal}(\phi)$ can be regarded as a benchmark of unbiased loss function, even though it is infeasible due to the missingness of $\{ r_{u,i}(1): o_{u,i} = 0\}$. 
  As such, a variety of methods are developed through approximating $\cL_{ideal}(\phi)$ to address the selection bias, in which the propensity-based estimators
  show the 
relatively superior performance~\citep{Schnabel-Swaminathan2016, Wang-Zhang-Sun-Qi2019}, and the IPS and DR estimators are 
    \begin{align*}
    \cE_{IPS} = |\cD|^{-1} \sum_{(u,i) \in \cD}  \frac{ o_{u,i}e_{u,i} }{ \hat p_{u, i} } \quad \text{and} \quad
         \cE_{DR} = |\cD|^{-1} \sum_{(u,i) \in \cD} \Big [ \hat e_{u,i}  +  \frac{ o_{u,i} (e_{u,i} -  \hat e_{u,i}) }{ \hat p_{u, i} } \Big ], 
     \end{align*}   
where $\hat p_{u,i}$ is an estimate of propensity score $p_{u,i} := \P(o_{u,i}=1 | x_{u,i})$, $\hat e_{u,i}$ is an estimate of $e_{u,i}$. 
 



\subsection{Related Work}

 {\bf Debiased learning in recommendation.}  The data collected in RS suffers from various biases~\citep{Chen-etal2020, Wu-etal2022}, which are entangled with the true preferences of users and pose a great  challenge to unbiased learning.  
 There is increasing interest in coping with different biases in recent years~\citep{zhang2021causal, ai2018unbiased, DBLP:conf/recsys/LiuCY16,  liu2021mitigating}.
 \citet{Schnabel-Swaminathan2016} proposed using inverse propensity score (IPS) and self-normalized IPS (SNIPS) methods to address the selection bias on data missing not at random, \citet{saito2019unbiased} and \citet{saito2020ips} extended them to implicit feedback data. 
 \citet{marlin2012collaborative} and \citet{steck2013evaluation} 
 derived an error imputation-based (EIB) unbiased learning method. These three approaches adopt two-phase learning~\citep{Wang-etal2021}, which first learns a propensity/imputation model and then applies it to construct an unbiased estimator of the ideal loss to train the recommendation model.   
 A doubly robust joint learning (DR-JL)  method~\citep{Wang-Zhang-Sun-Qi2019} was proposed by combining the IPS and EIB approaches. 
 Subsequently, strands of enhanced joint learning methods were developed, including MRDR~\citep{MRDR}, Multi-task DR~\citep{Zhang-etal2020}, DR-MSE~\citep{Dai-etal2022}, BRD-DR~\citep{Ding-etal2022}, TDR~\cite{TDR},  
  uniform data-aware methods~\citep{bonner2018causal,liu2020general, Chen-etal2021, Wang-etal2021, Balance2023} that aimed to seek better recommendation strategies by leveraging a small uniform dataset, and  
 multiple robust method~\citep{MR2023} that specifies multiple propensity and imputation models and achieves unbiased learning    
  if any of the propensity models, imputation models, or even a linear combination of these models can accurately estimate the true propensities or prediction errors. 
 \citet{Chen-etal2020} reviewed various biases in RS and discussed the recent progress on debiasing tasks. 
  \citet{Wu-etal2022} established the connections between the biases in causal inference and the biases, thereby presenting the formal causal definitions for RS. 




{\bf Stabilized causal effect estimation.}  The proposed method builds on the stabilized average causal effect estimation approaches in causal inference. 
\cite{Molenberghs-etal2015} summarized the limitations of doubly robust methods, including unstable 
 to small propensities~\citep{Kang-Schafer-2007, Wu-Han2022},  unboundedness~\citep{Laan-Rose2011}, and large variance~\citep{Tan-2007}. 
 These issues inspired a series of stabilized causal effect estimation methods in statistics~\citep{Kang-Schafer-2007, Bang-Robins-2005, Laan-Rose2011, Molenberghs-etal2015}.  
 Unlike previous works that focused only on achieving learning with unbiasedness in RS, this paper provides a new perspective to develop doubly robust estimators with much more stable statistical properties.

\section{Stabilized Doubly Robust Estimator}
In this section, we elaborate the limitations of DR methods and propose a stabilized DR (SDR) estimator with a weaker reliance on extrapolation. Theoretical analysis shows that SDR has bounded bias and generalization error bound for arbitrarily small propensities, while IPS and DR don't.

\subsection{Motivation} 


Even though DR estimator has double robustness property, its performance could be significantly improved if the following three stabilization aspects can be enhanced.  


{\bf More stable to small propensities.} As shown in~\citet{Schnabel-Swaminathan2016},  \citet{Wang-Zhang-Sun-Qi2019} and \citet{MRDR}, if there exist some extremely small estimated propensity scores, the IPS/DR estimator and its bias, variance, and tail bound are unbounded, 
 deteriorating the prediction accuracy. 
     What's more, such problems are widespread in practice, given the fact that there are many long-tailed users and items in RS, resulting in the presence of extreme propensities.   

{\bf More stable through weakening extrapolation.} DR relies more on extrapolation because the imputation model in DR is learned from the exposed events $\mathcal{O}$ and extrapolated to the unexposed events.
If the distributional disparity of $e_{u,i}$ on $o_{u,i}=0$ and $o_{u,i}=1$ is large, the imputed errors 
 are likely to be inaccurate on the unexposed events and incur bias of DR. Therefore, it is beneficial to reduce bias if we can develop an enhanced DR method with weaker reliance on extrapolation. 

{\bf More stable training process of updating a prediction model.} In general, alternating training between models results in better performance. From Figure \ref{fig:sdr}, \citet{Wang-Zhang-Sun-Qi2019} proposes joint learning for DR, alternatively updating the error imputation and prediction models with given estimated propensities. Double learning~\citep{MRDR} further incorporates parameter sharing between the imputation and prediction models. 
Bi-level optimization~\citep{Wang-etal2021, Chen-etal2021} can be viewed as alternately updating the prediction model and the other parameters used to estimate the loss. To the best of our knowledge, this is the first paper that proposes a algorithm to update the three models (i.e., error imputation model, propensity model, and prediction model) separately using different optimizers, which may resulting in more stable and accurate rating predictions.
 

\subsection{Stabilized Doubly Robust Estimator}\label{sec:3.2}




We propose a stabilized doubly robust (SDR) estimator that has a weaker dependence on extrapolation and is robust to small propensities. 
The SDR estimator consists of the following three steps.

{\bf Step 1 (Initialize imputed errors).} Pre-train imputation model $\hat e_{u,i}$, let $\hat \cE \triangleq {|\cD|}^{-1} \sum_{(u,i)\in \cD} \hat e_{u,i}$.

{\bf Step 2 (Learn constrained propensities).} Learn a propensity model $\hat p_{u,i}$ satisfying 
\begin{align} \label{step2}
\frac{1}{|\cD|} \sum_{(u,i)\in  \cD} {\frac{o_{u,i}}{\hat p_{u,i}}}{\left(\hat e_{u,i}-\hat \cE\right)}	  = 0.
\end{align}


{\bf Step 3 (SDR estimator).} The SDR estimator is given as
\begin{align*}
\cE_{SDR} 
= \sum_{(u,i)\in  \cD} \frac{o_{u,i} e_{u,i}}{\hat p_{u,i}}\Big / \sum_{(u,i)\in  \cD} \frac{o_{u,i}}{\hat p_{u,i}} \triangleq \sum_{(u,i)\in  \cD}  w_{u,i} e_{u,i}, 
\end{align*}
where $w_{u,i} =\frac{o_{u,i}}{\hat p_{u,i}}\big / \sum_{(u,i)\in  \cD} \frac{o_{u,i}}{\hat p_{u,i}}$. It can be seen that SDR estimator has the 
same form as SNIPS estimator, but the propensities are learned  differently. In SDR, the estimation of propensity model relies on the imputed errors, whereas not in SNIPS.


Each step in the construction of SDR estimator 
plays a different role. Specifically, the Step 2 is designed to enable double robustness property as shown in Theorem \ref{th1} (see Appendix \ref{app-th1} for proofs).   
\begin{theorem}[Double Robustness] 
\label{th1} 
$\cE_{SDR}$ is an asymptotically unbiased\footnote{Asymptotically unbiased means unbiasedness as the sample size goes to infinity.} estimator of $\cL_{ideal}$, 
when either the learned propensities $\hat p_{u,i}$ or the imputed errors $\hat e_{u,i}$ are accurate for all user-item pairs.
\end{theorem} 


We provide an intuitive way to illustrate the rationale of SDR.  
{\bf On the one hand}, if the propensities can be accurately estimated (i.e., $\hat p_{u,i} = p_{u,i}$) by using a
common model (e.g., logistic regression) without imposing constraint (\ref{step2}). Then the expectation of the left hand side of constraint (\ref{step2}) becomes  
		$$
			\mathbb{E}_{\mathcal{O}} \Big [  \frac{1}{|\mathcal{D}|} \sum_{(u,i)\in  \mathcal{D}} \frac{o_{u,i}}{\hat p_{u,i}}  \left ( \hat e_{u,i}- \mathcal{\hat E}  \right  ) \Big	] =   \frac{1}{|\mathcal{D}|} \sum_{(u,i)\in  \mathcal{D}}  \left ( \hat e_{u,i}- \mathcal{\hat E}  \right  ) \equiv 0,
		$$
	which indicates the constraint (\ref{step2}) always holds as the sample size goes to infinity by the strong law of large numbers\footnote{This is the reason why we adopt the notation of "asymptotically unbiased".}, irrespective of the accuracy of  the imputed errors $\hat e_{u,i}$.  	
This implies that the constraint (\ref{step2}) \emph{imposes almost no restriction} on the estimation of propensities. In this case, the SDR estimator is almost equivalent to the original SNIPS estimator.  
{\bf On the other hand}, if the propensities \emph{cannot} be accurately estimated by using a common 
model, but the imputed errors are accurate (i.e., $\hat e_{u,i} = e_{u,i}$).  In this case,  $\mathcal{\hat E}$ is an unbiased estimator. Specifically, $\cE_{SDR}$ satisfies 
\begin{align}\label{sdr}
\frac{1}{|\cD|} \sum_{(u,i)\in  \cD} \frac{o_{u,i}(e_{u,i}-\cE_{SDR})}{\hat p_{u,i}} = 0.
\end{align} 
Combining the constraint (\ref{step2}) and equation (\ref{sdr}) gives
\begin{align}  \label{eq5} 
\frac{1}{|\cD|} \sum_{(u,i)\in  \cD} \left[
\frac{o_{u,i}(e_{u,i}-\hat e_{u,i})}{\hat p_{u,i}} + \frac{o_{u,i}(\hat  \cE-\cE_{SDR})}{\hat p_{u,i}}
\right] = 0,
\end{align}  
where the first term equals to 0 if $\hat e_{u,i} = e_{u,i}$, 
it implies that 
$\cE_{SDR} =\hat \cE$, then the unbiasedness of $\cE_{SDR}$ follows immediately from the unbiasedness of $\hat \cE$.
 

In addition, Step 3 is designed for two main reasons to achieve stability. First, $\mathcal{E}_{SDR}$ is more robust to extrapolation compared with DR. This is because the propensities are learned from the entire data and thus have less requirement on extrapolation.
Second, $\mathcal{E}_{SDR}$ is more stable to small propensities, since the self-normalization imposes the weight $w_{u,i}$ to fall on the interval [0,1].

 In summary, forcing the propensities to satisfy the constraint (\ref{step2}) makes the SDR estimator not only doubly robust, but also captures the advantages of both SNIPS  and DR estimators. The design of SDR enables the constrained propensities to adaptively 
 find the direction of debiasing if either the learned  propensities without imposing constraint (\ref{step2}) or the imputed errors are accurate.

\subsection{Theoretical Analysis of Stableness}
Through theoretical analysis, we note that previous debiasing estimators such as  
IPS~\citep{Schnabel-Swaminathan2016} and DR-based methods~\citep{Wang-Zhang-Sun-Qi2019, MRDR} tend to have infinite biases, variances, tail bound, and corresponding generalization error bounds, in the presence of extremely small estimated propensities. 
Remarkably, the proposed SDR estimator doesn't suffer from such problems and is stable to arbitrarily small propensities, as shown in the following Theorems (see Appendixes \ref{app-a}, \ref{secondorder} and \ref{app-b} for proofs).



\begin{theorem}[Bias of SDR] \label{thbias} Given imputed errors $\hat e_{u,i}$ and learned propensities $\hat p_{u,i}$ satisfying the stabilization constraint  (\ref{step2}), with $\hat p_{u,i}>0$ for all user-item pairs, the 
bias 
of $\cE_{SDR}$ is 
\[ 
{\operatorname{Bias}}(\cE_{SDR}) = \Biggl |\frac{1}{|\cD|} \sum_{(u,i)\in  \cD} \left ( \delta_{u, i}-\frac{\sum_{(u,i)\in  \cD}{ \delta_{u, i}p_{u,i}}/\hat p_{u,i}}{\sum_{(u,i)\in  \cD}{p_{u,i}}/\hat p_{u,i}} \right ) \Biggr | + O( |\cD|^{-1} ), 
\]
where $\delta_{u, i}= e_{u, i}- \hat e_{u, i}$ is the error deviation.
\end{theorem}

Theorem \ref{thbias} shows the bias of the SDR estimator consisting of a dominant term given by the difference between $\delta_{u,i}$ and its weighted average, and a negligible term of order $O(|\cD|^{-1})$. The fact that the $\delta_{u,i}$ and its convex combinations are bounded, shows that the bias is bounded for arbitrarily small $\hat p_{ u, i}$. Compared to the $\operatorname{Bias}\left(\mathcal{E}_{IPS}\right)={|\mathcal{D}|}^{-1}|\sum_{u, i \in \mathcal{D}} {(\hat{p}_{u, i}-p_{u, i})e_{u, i}}/{\hat{p}_{u, i}} |$ and 
$\operatorname{Bias}\left(\mathcal{E}_{DR}\right)={|\mathcal{D}|}^{-1}|\sum_{u, i \in \mathcal{D}} {(\hat{p}_{u, i}-p_{u, i})\delta_{u, i}}/{\hat{p}_{u, i}} |$, it indicates that IPS and DR will have extremely large bias when there exists an extremely small $\hat p_{u,i}$.

\begin{theorem}[Variance of SDR]\label{thvar}
\label{thm-var}Under the conditions of Theorem \ref{thbias}, the variance of $\cE_{SDR}$ is
\[
\operatorname{Var}\left(\mathcal{E}_{SDR}\right)=\frac{\sum_{(u,i) \in \cD} p_{u,i}(1-p_{u,i})h^2_{u,i}/\hat p^2_{u,i}}{\left(\sum_{(u,i) \in \cD}p_{u,i}/\hat p_{u,i}\right)^2}+O( |\cD|^{-2} ),
\]
where $h_{u,i}=(e_{u,i}-\hat e_{u,i})-\sum_{(u,i) \in \cD}\{p_{u,i}(e_{u,i}-\hat e_{u,i})/\hat p_{u,i}\}/\sum_{(u,i) \in \cD}\{p_{u,i}/\hat p_{u,i}\}$ is a bounded difference between $e_{u,i}-\hat e_{u,i}$ and its weighted average. 
\end{theorem}
Theorem \ref{thvar} shows the variance of the SDR estimator consisting of a dominant term and a negligible term of order $O(|\cD|^{-2})$. The boundedness of the variance for arbitrarily small $\hat p_{ u, i}$ is given directly from the fact that SDR has a bounded range given by the self-normalized form. 
Compared to the $\operatorname{Var}\left(\mathcal{E}_{IPS}\right)={|\mathcal{D}|}^{-2}\sum_{u, i \in \mathcal{D}} {{p}_{u, i}(1-p_{u, i})e^2_{u, i}}/{\hat{p}^2_{u, i}}$ and 
$\operatorname{Var}\left(\mathcal{E}_{DR}\right)={|\mathcal{D}|}^{-2}\sum_{u, i \in \mathcal{D}} {{p}_{u, i}(1-p_{u, i})(e_{u, i}-\hat e_{u, i})^2}/{\hat{p}^2_{u, i}} $, it indicates that IPS and DR will have extremely large variance (tend to infinity) when there exist an extremely small $\hat p_{u,i}$ (tends to 0).
\begin{theorem}[Tail Bound of SDR] \label{th4} 
{Under the conditions of Theorem \ref{thbias},} 
for any prediction model, with probability $1-\eta$, the deviation of $\cE_{SDR}$ from its expectation has the following tail bound 
\[ 
\left|\cE_{SDR}-\bfE_{{\cO}} (\cE_{SDR})\right| \leq\sqrt{\frac{1}{2}{\log \left(\frac{4}{\eta}\right)} \sum_{ (u, i) \in \mathcal{D}}\frac{(\delta_{\operatorname{max}}-\delta_{u, i})^2+(\delta_{u, i}-\delta_{\operatorname{min}})^2}{\{1+\hat p_{u,i}(\sum_{\cD\setminus(u,i)}p_{u,i}/\hat p_{u,i}-\epsilon^{\prime})\}^2}}
\]  
where $\delta_{\operatorname{min}}=\operatorname{min}_{(u,i)\in \cD}\delta_{u, i}$, $\delta_{\operatorname{max}}=\operatorname{max}_{(u,i)\in \cD}\delta_{u, i}$,  $\epsilon^{\prime}={\small \sqrt{ \log(4/\eta) / 2 \cdot   \sum_{\cD\setminus(u,i)} 1/ \hat p^2_{u,i}}}$, and {$\cD\setminus(u,i)$ is the set of $\cD$ excluding the element $(u,i)$.}  
\end{theorem}
Note that $\sum_{\cD\setminus(u,i)}p_{u,i}/\hat p_{u,i}=O( |\cD|)$ and $\epsilon^{\prime}=O( |\cD|^{1/2} )$ in Theorem \ref{th4}, it follows that the tail bound of the SDR estimator converges to 0 for large samples. In addition, the tail bound is bounded for arbitrarily small $\hat p_{ u, i}$. Compared to the tail bound of IPS and DR, with probability $1-\eta$, we have
\begin{gather*}
    \left|\mathcal{E}_{{IPS}}-\mathbb{E}_{\cO}\left[\mathcal{E}_{{IPS}}\right]\right| \leq \sqrt{\frac{\log \left(2/\eta\right)}{2|\mathcal{D}|^{2}} \sum_{(u, i) \in \mathcal{D}}\left(\frac{e_{u, i}}{\hat{p}_{u, i}}\right)^{2}}, 
\left|\mathcal{E}_{{DR}}-\mathbb{E}_{\cO}\left[\mathcal{E}_{{DR}}\right]\right| \leq \sqrt{\frac{\log \left(2/\eta\right)}{2|\mathcal{D}|^{2}} \sum_{(u, i) \in \mathcal{D}}\left(\frac{\delta_{u, i}}{\hat{p}_{u, i}}\right)^{2}}, 
\end{gather*}which are both unbounded when $\hat p_{u,i} \rightarrow 0$. For SDR in the prediction model training phase, the boundedness of the generalization error bound  (see Theorem \ref{th5} in Appendix \ref{app-e}) 
follows immediately from the boundedness of the bias and tail bound.  
The above analysis demonstrates that SDR can comprehensively mitigate the negative effects caused by extreme propensities and results in more stable predictions.
Theorems \ref{thbias}-\ref{th5} are stated under the constraint $(\ref{step2})$. If we estimate the propensities with constraint (\ref{step2}), but finally constraint (\ref{step2}) somehow doesn't hold exactly, the associated bias, variance, and generalization error bound of SDR are presented in Appendix \ref{app-f}.

\section{Cycle Learning with Stabilization} 
In this section, we propose a novel SDR-based cycle learning approach, that not only exploits the stable statistical properties of the SDR estimator itself, but also carefully designs the updating process among various models to achieve higher stability.
In general, inspired by the idea of value iteration in reinforcement learning~\citep{sutton2018reinforcement}, alternatively updating the model tends to achieve better predictive performance, as existing debiasing training approaches suggested~\citep{Wang-Zhang-Sun-Qi2019, MRDR, Chen-etal2021}. As shown in Figure \ref{fig:sdr}, the proposed approach dynamically interacts with three models, utilizing the propensity model and imputation model simultaneously in a differentiated way, which can be regarded as an extension of these methods. In cycle learning, given pre-trained propensities, the inverse propensity weighted imputation error loss is used to first obtain an imputation model, and then take the constraint (\ref{step2}) as the regularization term to train a stabilized propensity model and ensure the double robustness of SDR. Finally, the prediction model is updated by minimizing the SDR 
loss and used to readjust the imputed errors. By repeating the above update {processes} cyclically, the cycle learning approach can fully utilize and combine the advantages of the three models to achieve more accurate rating predictions.  

Specifically, the data MNAR leads to the presence of missing $r_{u,i}(1)$, so that all $e_{u,i}$ cannot be used directly. Therefore, we obtain imputed errors by learning a pseudo-labeling model $\Tilde r_{u,i}(1)$ parameterized by $\beta$, and the imputed errors $\hat e_{u,i} =\mathrm{CE}(\Tilde r_{u,i}(1), \hat r_{u,i}(1))$ are updated by minimizing
\begin{align*}  
 \cL_{e}\left(\phi, \alpha, \beta \right)=|\cD|^{-1} \sum_{ (u, i) \in \cD} \frac{o_{u,i} ( \hat{e}_{u, i}-e_{u, i})^{2}}{\pi(x_{u,i}; \alpha)}+\lambda_e \|\beta\|_{F}^{2},
\end{align*} 
where ${e}_{u, i}=\mathrm{CE}(r_{u,i}(1), \hat r_{u,i}(1))$, $\lambda_e\geq 0$,  $\hat  p_{u,i}=\pi(x_{u,i}; \alpha)$ is the propensity model, $\|\cdot\|_{F}^{2}$ is the Frobenius norm. For each observed ratings, the inverse of the estimated propensities are used for weighting to account for MNAR effects. Next, we consider two methods for estimating propensity scores, which are Naive Bayes with {Laplace smoothing} and logistic regression. {The former provides a wide range of opportunities for achieving stability constraint (\ref{step2}) through the selection of smoothing coefficients.} The latter requires user and item embeddings, which are obtained by employing MF before performing cycle learning. The learned propensities need to both satisfy the accuracy, which is evaluated with cross entropy, and meet the constraint (\ref{step2}) for stabilization and double robustness. 
The propensity model {$\pi(x_{u,i}; \alpha)$} is updated by using the loss $\cL_{ce}\left(\phi, \alpha, \beta \right)+\eta \cdot \cL_{{stable}}\left(\phi, \alpha, \beta \right)$, where $\cL_{ce}\left(\phi, \alpha, \beta \right)$ is cross entropy loss of propensity model and
\begin{align*}
\cL_{{stable}}\left(\phi, \alpha, \beta \right)= |\cD|^{-1} \Big \{\sum_{(u,i)\in  \cD} \frac{o_{u,i}}{\pi(x_{u,i};  \alpha)}{\left(\hat e_{u,i}-\hat \cE\right)}\Big\}^2+ \lambda_{stable} \|\alpha\|_{F}^{2},
\end{align*}
where {$\lambda_{stable}\geq 0$}, and $\eta$ is a hyper-parameter for trade-off. Finally, the prediction model $f(x_{u,i}; \phi)$ is updated by minimizing the SDR loss
\begin{align*}
\cL_{sdr}\left(\phi, \alpha, \beta \right) = \Big [{
 \sum_{(u,i)\in  \cD}  
\frac{o_{u,i}e_{u,i}}{\pi(x_{u,i}; \alpha)}
}\Big ]\Big /\Big [{ \sum_{(u,i)\in  \cD} 
\frac{o_{u,i}}{\pi(x_{u,i}; \alpha)}
}\Big ] + \lambda_{sdr} \|\phi\|_{F}^{2},
\end{align*}
where the first term is equivalent to the left hand side of equation (\ref{eq5}), and $\lambda_{sdr}\geq 0$. In cycle learning, the updated prediction model will be used for re-update the imputation model using the next sample batch. Notably, the designed algorithm strictly follows the proposed SDR estimator in Section \ref{sec:3.2}. From Figure \ref{fig:sdr} and Alg. \ref{alg1}, our algorithm first updates imputed errors $\hat e$ by Step 1, and then learns a propensity $\hat p$ based on learned $\hat e$ to satisfy the constraint (\ref{step2}) in Step 2. The main purpose of the first two steps is to ensure that the SDR estimator in Step 3 has double robustness and has a lower extrapolation dependence compared to the previous DR methods. Finally, from Step 3 we update the predicted rating $\hat r$ by minimizing the estimation of the ideal loss using the proposed SDR estimator. For the next round, instead of re-initializing, Step 1 updates the imputed errors $\hat e$ according to the new prediction model, then Step 2 re-updates the constrained propensities $\hat p$, and then uses Step 3 to update the prediction model $\hat r$ again, and so on. We summarized the cycle learning approach in Alg. \ref{alg1}.

\begin{algorithm}[t]
\caption{The Proposed Stable DR (MRDR) Cycle Learning, Stable-DR (MRDR)}
\label{alg1}
\LinesNumbered
\KwIn{observed ratings $\mathbf{Y}^{o}$, and $\eta, \lambda_{e}, \lambda_{stable}, \lambda_{sdr}\geq 0$}
\While{stopping criteria is not satisfied}{
    \For{number of steps for training the imputation model}{Sample a batch of user-item pairs $\left\{\left(u_{j}, i_{j}\right)\right\}_{j=1}^{J}$ from $\mathcal{O}$\;
    Update $\beta$ by descending along the gradient $\nabla_{\beta} \cL_{e}\left(\phi, \alpha, \beta \right)$\;
    }
    \For{number of steps for training the propensity model}
    {
    Sample a batch of user-item pairs $\left\{\left(u_{k}, i_{k}\right)\right\}_{k=1}^{K}$ from $\mathcal{D}$\;
    Calculate the gradient of propensity cross entropy error $\nabla_\alpha \cL_{{ce}}\left(\phi, \alpha, \beta \right)$\;
    Calculate the gradient of propensity stable {constraint (\ref{step2})} $\nabla_\alpha \cL_{stable}\left(\phi, \alpha, \beta \right)$\;
    Update $\alpha$ by descending along the gradient $\nabla_\alpha \cL_{ce}\left(\phi, \alpha, \beta \right)+\eta\cdot \nabla_\alpha \cL_{{stable}}\left(\phi, \alpha, \beta \right)$}
    \For{number of steps for training the prediction model}{
    Sample a batch of user-item pairs $\left\{\left(u_{l}, i_{l}\right)\right\}_{l=1}^{L}$ from $\mathcal{O}$\;
    Update $\phi$ by descending along the gradient $\nabla_{\phi} \cL_{sdr}\left(\phi, \alpha, \beta \right)$\;}
}
\end{algorithm}

\vspace{-6pt}
\section{Real-world Experiments}
\vspace{-6pt}
In this section, several experiments are conducted to evaluate the proposed methods on two real-world benchmark datasets. We conduct experiments to answer the following questions:
\begin{enumerate}
	\item[\bf RQ1.]  Do the proposed Stable-DR and Stable-MRDR approaches improve in debiasing performance compared to the existing studies?
	\item[\bf RQ2.]  Do our methods stably perform well under the various propensity models?

 \item[\bf RQ3.]  How does the performance of our method change under different strengths of the stabilization constraint?
	\end{enumerate}
\subsection{Experimental Setup}
{\bf Dataset and preprocessing.} To answer the above RQs, we need to use the datasets that contain both MNAR ratings and missing-at-random (MAR) ratings. Following the previous studies~\citep{Schnabel-Swaminathan2016, Wang-Zhang-Sun-Qi2019, MRDR, Chen-etal2021}, we conduct experiments on the two commonly used datasets: {\bf Coat\footnote{https://www.cs.cornell.edu/\textasciitilde schnabts/mnar/}} contains ratings from 290 users to 300 items. Each user evaluates 24 items, containing 6,960 MNAR ratings in total. Meanwhile, each user evaluates 16 items randomly, which generates 4,640 MAR ratings. {\bf Yahoo! R3\footnote{http://webscope.sandbox.yahoo.com/}} contains totally 311,704 MNAR and 54,000 MAR ratings from 15,400 users to 1,000 items.
\begin{table}[t]
 \centering
 \small
    \setlength{\tabcolsep}{5pt}
    \vspace{-6pt}
    \captionof{table}{Performance on Coat and Yahoo!R3, using MF, SLIM, and NCF as the base models.}
    \begin{tabular}{l|cccc|cccc}
\toprule
Dataset   & \multicolumn{4}{c|}{Coat}           & \multicolumn{4}{c}{Yahoo!R3}          \\
\cmidrule(r){1-1}  \cmidrule(lr){2-5} \cmidrule(lr){6-9}
Method    & MSE    & AUC    & N@5 & N@10 & MSE    & AUC    & N@5 & N@10 \\
          \midrule
MF        & 0.2428 & 0.7063 & 0.6025 & 0.6774  & 0.2500 & 0.6722 & 0.6374 & 0.7634  \\
+ IPS                & 0.2316 & 0.7166 & 0.6184 & 0.6897  & 0.2194 & 0.6742 & 0.6304 & 0.7556  \\
+ SNIPS              & 0.2333 & 0.7070 & 0.6222 & 0.6851  & \textbf{0.1931} & 0.6831 & 0.6348 & 0.7608  \\
+ AS-IPS            & \textbf{0.2121} & 0.7180 & 0.6160 & 0.6824  & 0.2391 & 0.6770 & 0.6364 & 0.7601  \\
+ CVIB                & 0.2195 & 0.7239 & 0.6285 & 0.6947  & 0.2625 & 0.6853 & 0.6513 & 0.7729  \\
\midrule
+ DR                 & 0.2298 & 0.7132 & 0.6243 & 0.6918  & 0.2093 & 0.6873 & 0.6574 & 0.7741  \\
+ DR-JL              & 0.2254 & 0.7209 & 0.6252 & 0.6961  & 0.2194 & 0.6863 & 0.6525 & 0.7701  \\
+ \textbf{Stable-DR (Ours)}   & 0.2159 & \textbf{0.7508} & \textbf{0.6511} & \textbf{0.7073}  & 0.2090 & \textbf{0.6946} & \textbf{0.6620} & \textbf{0.7786}  \\
\midrule
+ MRDR-JL            & 0.2252 & 0.7318 & 0.6375 & 0.6989  & 0.2173 & 0.6830 & 0.6437 & 0.7652  \\
+ \textbf{Stable-MRDR (Ours)} & \textbf{0.2076} & \textbf{0.7548} & \textbf{0.6532} & \textbf{0.7105}  & \textbf{0.2081} & \textbf{0.6915} & \textbf{0.6585} & \textbf{0.7757}  \\
\midrule
\midrule
SLIM      & {0.2419} & {0.7074} & {0.7064} & {0.7650}  & {0.2126} & {0.6636} & {0.7190} & {0.8134}  \\
{+ IPS}                & {0.2411} & {0.7058} & {0.7235} & {0.7644}  & {0.2046} & {0.6583} & {0.7285} & {0.8244}  \\
{+ SNIPS}              & {0.2420} & {0.7071} & {0.7369} & {0.7672}  & {0.2155} & {0.6720} & {0.7303} & {0.8227}  \\
{+ AS-IPS}             & {\textbf{0.2133}} & {0.7105} & {0.6238} & {0.6975}  & {\textbf{0.1946}} & {0.6769} & {0.6508} & {0.7702}  \\
{+ CVIB}                & {0.2413} & {0.7108} & {0.7214} & {0.7638}  & {\textbf{0.2024}} & {0.6790} & {0.7335} & {0.8221}  \\
\midrule
{+ DR}                 & {\textbf{0.2334}} & {0.7064} & {0.7267} & {0.7649}  & {0.2054} & {0.6771} & {0.7344} & {0.8248}  \\
{+ DR-JL}              & {0.2407} & {0.7090} & {0.7279} & {0.7655}  & {0.2044} & {0.6792} & {0.7360} & {0.8260}  \\
{+ \textbf{Stable-DR (Ours)}}   & {0.2356} & \textbf{{0.7201}} & \textbf{{0.7389}} & \textbf{{0.7724}}  & {0.2080} & {\textbf{0.6874}} & {\textbf{0.7473}} & {\textbf{0.8349}}  \\ \midrule
{+ MRDR-JL}            & {0.2409} & {0.7074} & {0.7329} & {0.7679}  & {0.2016} & {0.6791} & {0.7338} & {0.8239}  \\
{+ \textbf{Stable-MRDR (Ours)}} & {0.2369} & \textbf{{0.7148}} & \textbf{{0.7378}} & \textbf{{0.7711}}  & {0.2086} & {\textbf{0.6842}} & {\textbf{0.7435}} & {\textbf{0.8308}}  \\ \midrule \midrule
NCF       & 0.2050 & 0.7670 & 0.6228 & 0.6954  & 0.3215 & 0.6782 & 0.6501 & 0.7672  \\
+ IPS                & 0.2042 & 0.7646 & 0.6327 & 0.7054  & 0.1777 & 0.6719 & 0.6548 & 0.7703  \\
+ SNIPS              & 0.1904 & 0.7707 & 0.6271 & 0.7062  & 0.1694 & 0.6903 & 0.6693 & 0.7807  \\
+ AS-IPS             & 0.2061 & 0.7630 & 0.6156 & 0.6983  & 0.1715 & 0.6879 & 0.6620 & 0.7769  \\
+ CVIB                & 0.2042 & 0.7655 & 0.6176 & 0.6946  & 0.3088 & 0.6715 & 0.6669 & 0.7793  \\
\midrule
+ DR                 & 0.2081 & 0.7578 & 0.6119 & 0.6900  & 0.1705 & 0.6886 & 0.6628 & 0.7768  \\
+ DR-JL              & 0.2115 & 0.7600 & 0.6272 & 0.6967  & 0.2452 & 0.6818 & 0.6516 & 0.7678  \\
+ \textbf{Stable-DR (Ours)}   & \textbf{0.1896} & \textbf{0.7712} & \textbf{0.6337} & \textbf{0.7095}  & \textbf{0.1664} & \textbf{0.6907} & \textbf{0.6756} & \textbf{0.7861}    \\ \midrule
+ MRDR-JL            & 0.2046 & 0.7609 & 0.6182 & 0.6992  & 0.2367 & 0.6778 & 0.6465 & 0.7664  \\
+ \textbf{Stable-MRDR (Ours)} & \textbf{0.1899} & \textbf{0.7710} & \textbf{0.6380} & \textbf{0.7082}  & \textbf{0.1671} & \textbf{0.6910} & \textbf{0.6734} & \textbf{0.7846}  \\ \bottomrule
\end{tabular} \label{tab1}
\vskip -0.15in
\end{table}
{\bf Baselines.} In our experiments, we take {\bf Matrix Factorization (MF)}~\citep{koren2009matrix}, {\bf Sparse LInear Method (SLIM)}~\citep{ning2011slim}, and {\bf Neural Collaborative Filtering (NCF)}~\citep{he2017neural}
as the base model respectively, and compare against the proposed methods with the following baselines: {\bf Base Model}, {\bf IPS}~\citep{saito2020ips, Schnabel-Swaminathan2016}, {\bf SNIPS}~\citep{Swaminathan-Joachims2015}, {\bf IPS-AT}~\citep{saito2020asymmetric}, {\bf CVIB}~\citep{wang2020information}, 
{\bf DR}~\citep{saito2020doubly}, {\bf DR-JL}~\citep{Wang-Zhang-Sun-Qi2019}, and {\bf MRDR-JL}~\citep{MRDR}. In addition, {\bf Naive Bayes with Laplace smoothing} and {\bf logistic regression} are used to establish the propensity model respectively.

{\bf Experimental protocols and details.} The following four metrics are used simultaneously in the evaluation of debiasing performance: \emph{MSE, AUC, NDCG@5,} and \emph{NDCG@10}. All the experiments are implemented on PyTorch with Adam as the optimizer\footnote{For all experiments, we use NVIDIA GeForce RTX 3090 as the computing resource.}. We tune the learning rate in $\{0.005, 0.01, 0.05, 0.1\}$, weight decay in $\{1e-6, 5e-6, \dots, 5e-3, 1e-2\}$, constrain parameter eta in $\{50, 100, 150, 200\}$ for {\bf Coat} and $\{500, 1000, 1500, 2000\}$ for {\bf Yahoo! R3}, and batch size in $\{128, 256, 512, 1024, 2048\}$ for {\bf Coat} and $\{1024, 2048, 4096, 8192, 16384\}$ for {\bf Yahoo! R3}. 
In addition, for the Laplacian smooth parameter in Naive Bayes model, the initial value is set to 0 and the learning rate is tuned in $\{5, 10, 15, 20\}$ for {\bf Coat} and in $\{50, 100, 150, 200\}$ for {\bf Yahoo! R3}.

\subsection{Performance Comparison (RQ1)}
Table \ref{tab1} summarizes the performance of the proposed Stable-DR and Stable-MRDR methods compared with previous methods. First, the causally-inspired methods perform better than the base model, verifying the necessity of handling the selection bias in rating prediction. For previous methods, SNIPS, CVIB and DR demonstrate competitive performance. Second, the proposed Stable-DR and Stable-MRDR have the best performance in all four metrics. 
On one hand, our methods outperform SNIPS, attributed to the inclusion of the propensity model in the training process, as well as the boundedness and double robustness of SDR. On the other hand, our methods outperform DR-JL and MRDR-JL, attributed to the stabilization constraint introduced in the training of the propensity model. This further demonstrates the benefit of cycle learning, in which the propensity model is acted as the mediation between the imputation and prediction model during the training process, rather than updating the prediction model from the imputation model directly. 
\vspace{-4pt}
\subsection{Ablation and Parameter Sensitivity Study (RQ2, RQ3)}
\vspace{-2pt}
The debiasing performance under different stabilization constraint strength and propensity models is shown in Figure  \ref{fig:zt}. First, the proposed Stable-DR and Stable-MRDR outperform DR-JL and MRDR-JL, when either Naive Bayes with Laplace smoothing or logistic regression is used as propensity models. It indicates that our methods have better debiasing ability in both the feature containing and collaborative filtering scenarios. Second, when the strength of the stabilization constraint is zero, our method performs similarly to SNIPS and slightly worse than the DR-JL and MRDR-JL, which indicates that simply using cross-entropy loss to update the propensity model is not effective in improving the model performance. However, as the strength of the stabilization constraint increases, Stable-DR and Stable-MRDR using cycle learning have a stable and significant improvement compared to DR-JL and MRDR-JL. Our methods achieve the optimal performance at the appropriate constraint strength, which can be interpreted as simultaneous consideration of accuracy and stability to ensure boundedness and double robustness of SDR.

\begin{figure}[t]
    \centering
    \includegraphics[scale=0.35]{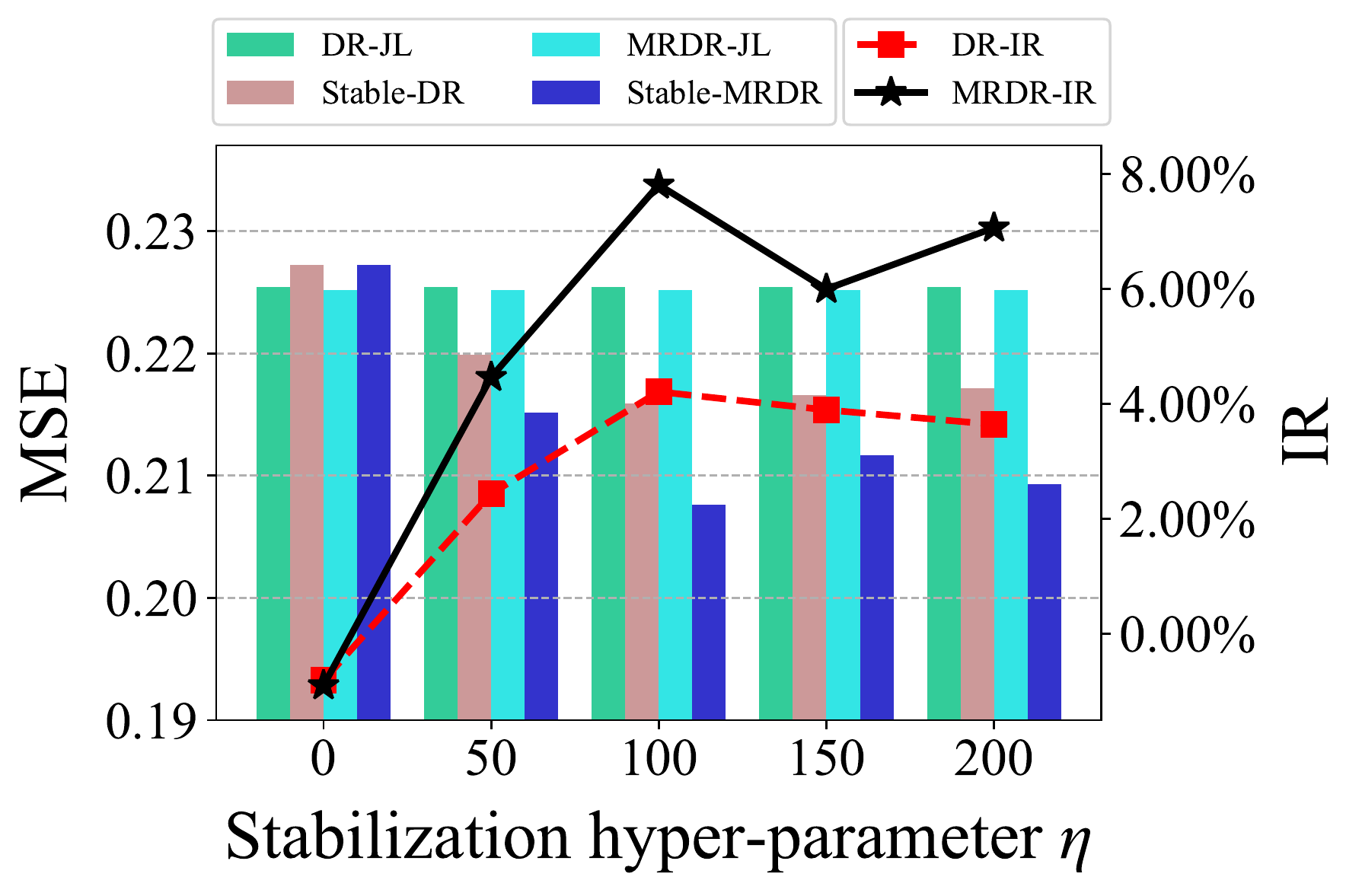} 
    \includegraphics[scale=0.35]{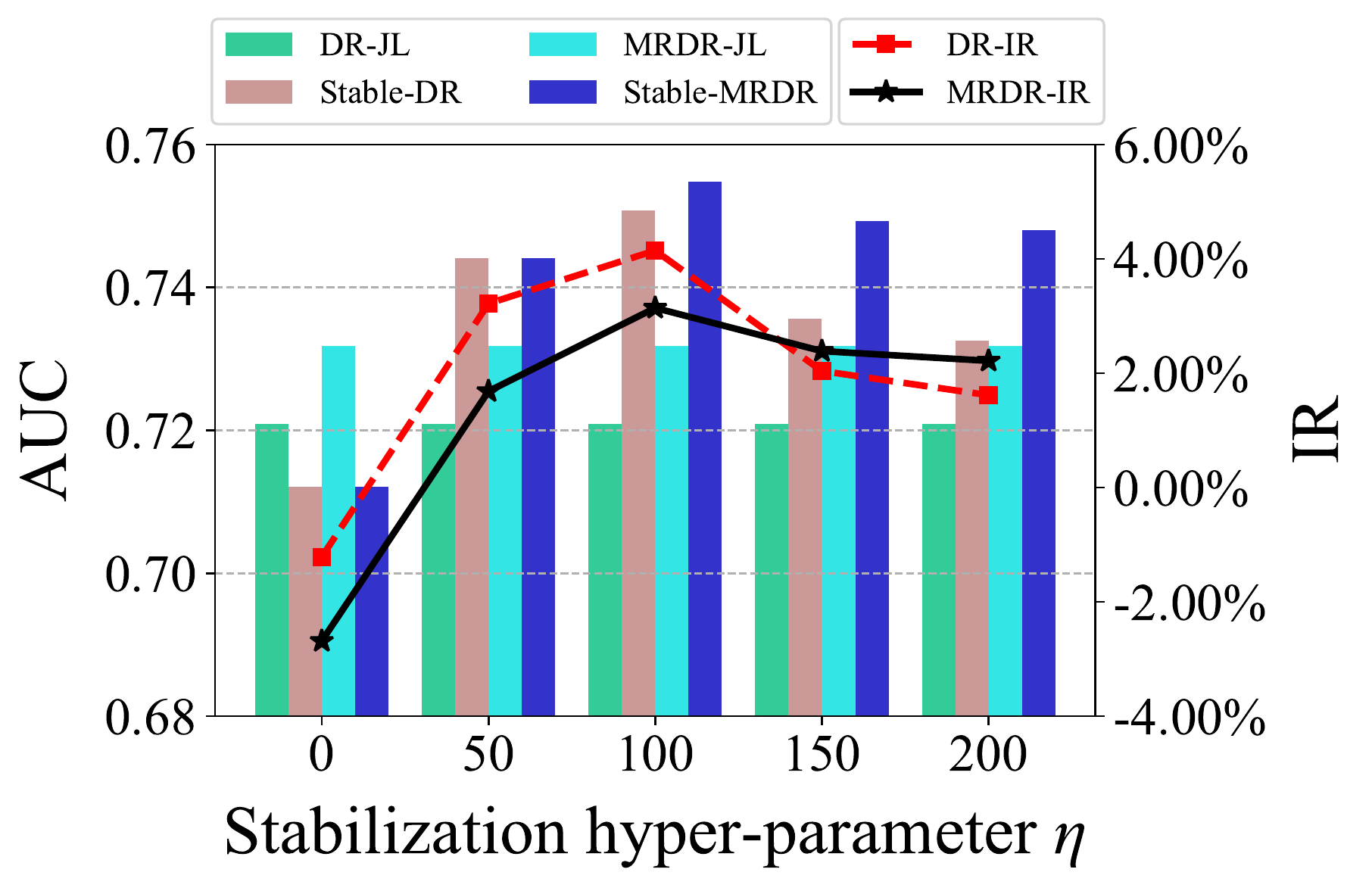} 
    \includegraphics[scale=0.35]{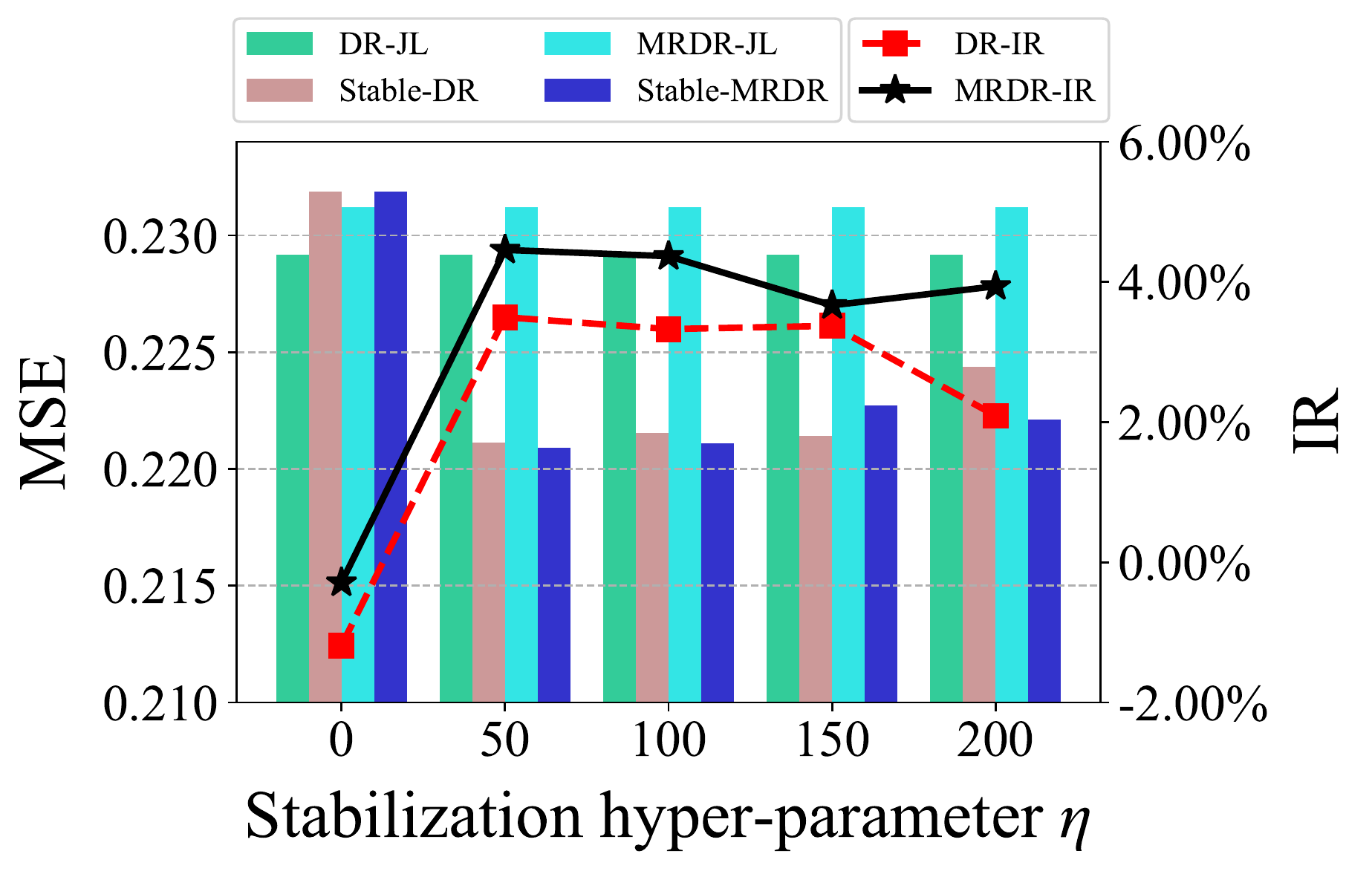} 
    \includegraphics[scale=0.35]{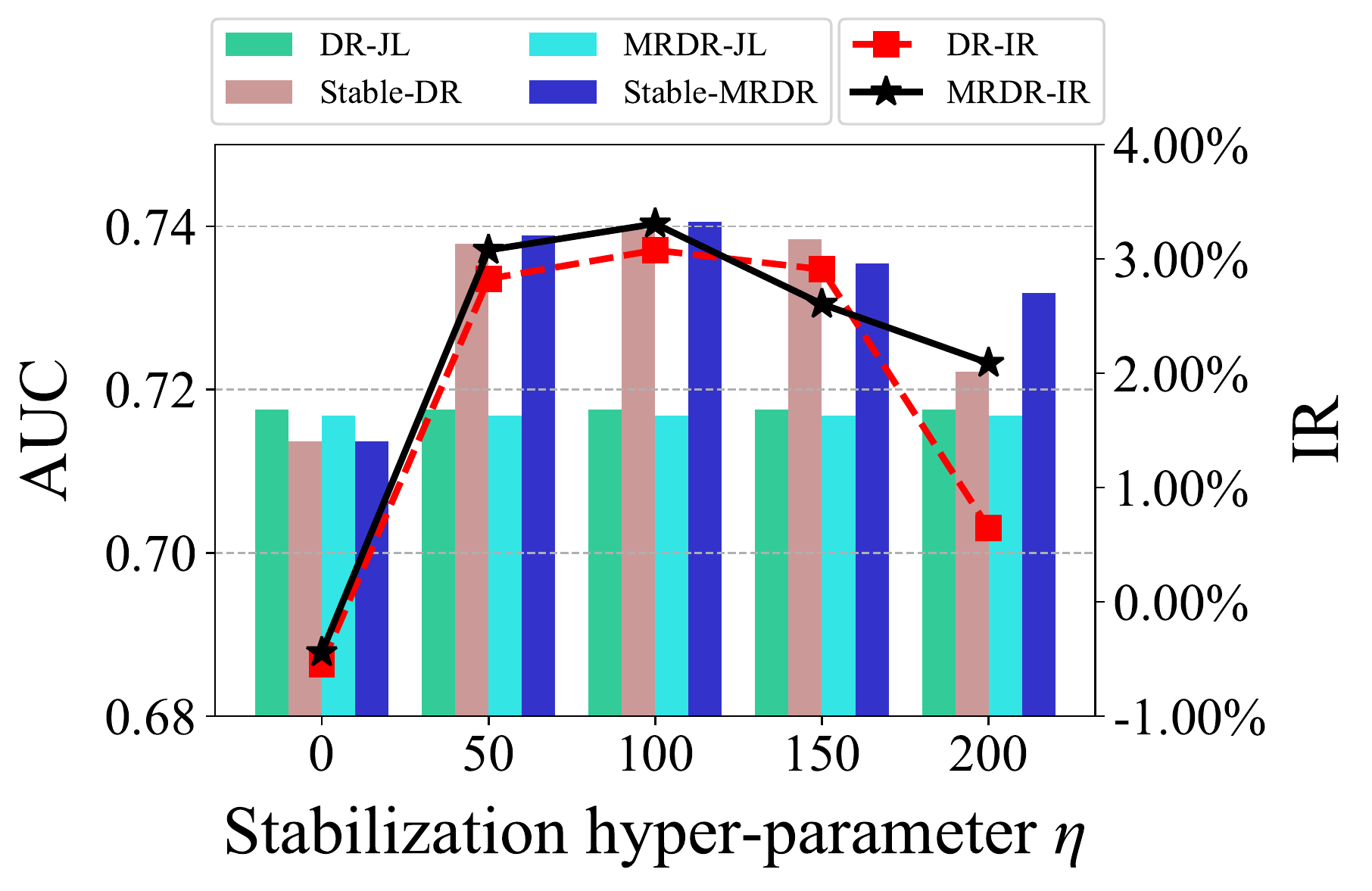} 
    \vspace{-6pt}
    \caption{MSE, AUC and Increasing Ratio (IR) of Stable-DR and Stable-MRDR comparing with two baseline algorithms DR-JL and MRDR-JL in two different propensity model setting: Naive Bayes with Laplace smoothing (Top) and logistic regression (Bottom) respectively.}
    \label{fig:zt}
    \vspace{-12pt}
\end{figure}

    \vspace{-11pt}
\section{Conclusion} 
    \vspace{-9pt}
In this paper, we propose an SDR estimator for data MNAR that maintains double robustness and improves the stability of DR in the following three aspects: first, we show that SDR has a weaker extrapolation dependence than DR and can result in more stable and accurate predictions in the presence of MNAR effects. Next, through theoretical analysis, we show that the proposed SDR has bounded bias, variance, and generalization error bounds  under inaccurate imputed errors and arbitrarily small estimated propensities, while DR does not. Finally, we propose a novel learning approach for SDR that updates the 
imputation, propensity, and prediction models cyclically, achieving more stable and accurate predictions. Extensive experiments show that our approach significantly outperforms the existing methods in terms of both convergence and prediction accuracy.

\section*{Ethics Statement} 
This work is mostly theoretical and experiments are based on synthetic and public datasets. We claim that this work does not present any foreseeable negative social impact.

\section*{Reproducibility Statement} 
Code is provided in Supplementary Materials to reproduce the experimental results.

\subsubsection*{Acknowledgments}
The work was supported by the National Key R\&D Program of China under Grant No. 2019YFB1705601.

\bibliography{iclr2023_conference}
\bibliographystyle{iclr2023_conference}



\newpage  
\appendix
\section*{Appendix}

Throughout, following existing studies~\citep{Schnabel-Swaminathan2016, Wang-Zhang-Sun-Qi2019, MRDR, Dai-etal2022}, we assume that the indicator matrix $\mathcal{O}$ contains independent random variables and each $o_{u, i}$ follows a Bernoulli distribution with probability $p_{u, i}$.

\section{Proof of theorems}
\subsection{Proof of Theorem 1} \label{app-th1}

\begin{proof}[Proof of Theorem 1]  
To demonstrate the double robustness of the SDR, first note that $P \left(\lim _{|\cD| \rightarrow \infty} \cE_{SDR}=\mathcal{L}_{ideal}\right)=1$ 
 if the learned propensities are accurate~\citep{Swaminathan-Joachims2015}, since $|\cD|^{-1} \sum_{(u,i)\in  \cD} o_{u,i}/\hat p_{u,i}$ converges to 1 almost surely as $|\cD|$ goes to infinity and IPS is unbiased. Besides, the constraint (\ref{step2}) is constructed to ensure the unbiasedness of $\cE_{SDR}$ if the error imputation model is correctly specified. In fact, 
$\cE_{SDR}$ satisfies 
\begin{align}\label{sdr2}
\frac{1}{|\cD|} \sum_{(u,i)\in  \cD} \frac{o_{u,i}(e_{u,i}-\cE_{SDR})}{\hat p_{u,i}} = 0.
\end{align} 
Combining the constraint (\ref{step2}) and equation (\ref{sdr2}) gives
\begin{align*} 
\frac{1}{|\cD|} \sum_{(u,i)\in  \cD} \left[
\frac{o_{u,i}(e_{u,i}-\hat e_{u,i})}{\hat p_{u,i}} + \frac{o_{u,i}(\hat  \cE-\cE_{SDR})}{\hat p_{u,i}}
\right] = 0,
\end{align*}  
where the first term equals to 0 when the imputation model is correctly specified, 
it implies that 
$\cE_{SDR} =\hat \cE$, then the unbiasedness of $\cE_{SDR}$ follows immediately from the unbiasedness of $\hat \cE$. 

\end{proof}

\subsection{Proof of Theorem 2} \label{app-a}
\begin{proof}[Proof of Theorem 2]  
Equation (\ref{eq5}) implies that   
$\cE_{SDR}$ can be expressed as 
\begin{equation}  \label{eq-s1}
\cE_{SDR}=\Big [{
\frac{1}{|\cD|} \sum_{(u,i)\in  \cD}  
\frac{o_{u,i}(e_{u,i}-\hat e_{u,i}+ \hat \cE)}{\hat p_{u,i}}
}\Big ]\Big /\Big [{\frac{1}{|\cD|} \sum_{(u,i)\in  \cD} 
\frac{o_{u,i}}{\hat p_{u,i}}
}\Big ].   
\end{equation} 
For notational simplicity, let 
\begin{align*}
 w_{u, i} \triangleq {o_{u,i}}/{\hat p_{u,i}} \quad \text{and} \quad v_{u, i} \triangleq  {o_{u,i}(e_{u,i}-\hat e_{u,i}+\hat \cE)}/{\hat p_{u,i}}, 
\end{align*}
then $\cE_{SDR}$ can be written as a ratio statistic   
\begin{align*}
\cE_{SDR} = {\frac{1}{|\cD|} \sum_{(u,i)\in  \cD}v_{u, i}} \Big / {\frac{1}{|\cD|} \sum_{(u,i)\in  \cD}w_{u, i}}\triangleq f(\bar v, \bar w),
\end{align*}
where $f( v,  w)= v / w$, $\bar v=|\cD|^{-1} {\sum_{(u,i)\in  \cD}v_{u, i}} $, and $\bar w= |\cD|^{-1}{\sum_{(u,i)\in  \cD}w_{u, 
i}}.$

Applying the Taylor expansion around $(\mu_v, \mu_w) \triangleq (\bfE[\bar v], \bfE[\bar w])$  yields  that 
\begin{align*}
f(\bar v, \bar w) &= f(\mu_{v}, \mu_{w})+f_{v}^{\prime}(\mu_{v}, \mu_{w})\left(\bar v-\mu_{v}\right)+f_{w}^{\prime}(\mu_{v}, \mu_{w})\left(\bar w-\mu_{w}\right) \\
&+\frac{1}{2}\left\{f_{vv}^{\prime \prime}(\mu_{v}, \mu_{w})\left(\bar v-\mu_{v}\right)^{2}+2 f_{v w}^{\prime \prime}(\mu_{v}, \mu_{w})\left(\bar v-\mu_{v}\right)\left(\bar w-\mu_{w}\right)+f_{w w}^{\prime \prime}\left(\bar w-\mu_{w}\right)^{2}\right\}  \\  
&+ R(\tilde v, \tilde w),
\end{align*}
where $R(\tilde v, \tilde w)$ is the remainder term. 
Note that $f_{vv}^{\prime \prime}(\mu_v, \mu_w) = 0$,  $f_{vw}^{\prime \prime}(\mu_v, \mu_w) =-1/\mu_w^{2}$, and $f^{\prime \prime}_{ww}(\mu_v, \mu_w) = 2 \mu_v / \mu_w^3$, then taking an expectation on both sides of the Taylor expansion leads to 
\[ 
\bfE(\bar v / \bar w)  = 
\frac{\mu_{v}}{\mu_{w}}-\frac{\operatorname{Cov}(\bar v, \bar w)}{\left(\mu_{w}\right)^{2}}+\frac{\operatorname{Var}(\bar w) \mu_{v}}{\left(\mu_{w}\right)^{3}} +  \bfE[R(\tilde v, \tilde w)].
\]
By some calculations, we have $\operatorname{Cov}(\bar v, \bar w) = O(|\cD|^{-1}),  \operatorname{Var}(\bar w) = O(|\cD|^{-1}),   \bfE[R(\tilde v, \tilde w)]  = o(|\cD|^{-1})$.  
Thus, the bias of $\cE_{SDR}$ is given as
$$
\begin{aligned}
{\operatorname{Bias}}(\cE_{SDR})&= \Biggl |\frac{1}{|\cD|} \sum_{(u,i)\in  \cD} \left ( \delta_{u, i}-\frac{\sum_{(u,i)\in  \cD}{ \delta_{u, i}p_{u,i}}/{\hat p_{u,i}}}{\sum_{(u,i)\in  \cD}{p_{u,i}}/{\hat p_{u,i}}} \right ) \Biggr | + O( |\cD|^{-1} ). 
\end{aligned}
$$


\end{proof}

\subsection{Proof of Theorem 3}\label{secondorder}
\begin{proof}[Proof of Theorem 3]

According to the proof of Theorem 2, 
we have 
    \begin{equation}  \label{eq-s2}
        \bfE[ \bar v / \bar w] - \mu_v/\mu_w  = O( |\cD|^{-1} ). 
    \end{equation}
Then the variance of $\cE_{SDR}$ can be decomposed into as 
    \begin{align*}
        \text{Var}( \cE_{SDR} ) ={}& \text{Var}( \bar v / \bar w  ) = \bfE \left [\left  \{ \bar v / \bar w  -  \bfE( \bar v / \bar w) \right  \}^2 \right ] \\
        ={}&  \bfE \left [\left  \{ \bar v / \bar w  -   \mu_v/\mu_w  \right  \}^2  - 2 O(|\cD|^{-1}) \cdot \left \{ \bar v / \bar w  -   \mu_v/\mu_w  \right  \} +   O(|\cD|^{-2})  \right ], \\
        ={}&  \cV_1 + \cV_2 +  O(|\cD|^{-2}),
    \end{align*}
 where $\cV_1 \triangleq \bfE  [  \{ \bar v / \bar w  -   \mu_v/\mu_w  \}^2]$, $\cV_2  \triangleq - 2 O(|\cD|^{-1}) \cdot [\bfE(\bar v / \bar w) - \mu_v/\mu_w]$. Equation (\ref{eq-s2}) implies that $\cV_2 = O(|\cD|^{-2})$.

Denote $f( v,  w)= v / w$, and apply delta method around $(\mu_v, \mu_w) \triangleq (\bfE[\bar v], \bfE[\bar w])$ to calculate $\cV_1$ yields that
\[
\begin{aligned}
\cV_1 ={}& \bfE\left\{\left[f(\mu_{v}, \mu_{w})+f_{v}^{\prime}(\mu_{v}, \mu_{w})\left(\bar v-\mu_{v}\right)+f_{w}^{\prime}(\mu_{v}, \mu_{w})\left(\bar w-\mu_{w}\right) +O_p( |\cD|^{-1} ) - f(\mu_{v}, \mu_{w})\right]^{2}\right\}\\
={}& \bfE\left\{\left[f_{v}^{\prime}(\mu_{v}, \mu_{w})\left(\bar v-\mu_{v}\right)+f_{w}^{\prime}(\mu_{v}, \mu_{w})\left(\bar w-\mu_{w}\right) +O_p( |\cD|^{-1} ) \right]^{2}\right\}\\
={}& \bfE\left\{f_{v}^{\prime 2}(\mu_{v}, \mu_{w})\left(\bar v-\mu_{v}\right)^{2}+2f_{v}^{\prime}(\mu_{v}, \mu_{w})\left(\bar v-\mu_{v}\right)f_{w}^{\prime}(\mu_{v}, \mu_{w})\left(\bar w-\mu_{w}\right)\right.\\
&+\left.f_{w}^{\prime 2}(\mu_{v}, \mu_{w})\left(\bar w-\mu_{w}\right)^{2}  \right\}+O( |\cD|^{-2} )\\
={}& f_{v}^{\prime 2}(\mu_{v}, \mu_{w})\operatorname{Var}(\bar v)+2f_{v}^{\prime}(\mu_{v}, \mu_{w})f_{w}^{\prime}(\mu_{v}, \mu_{w})\operatorname{Cov}(\bar v, \bar w)+ f_{w}^{\prime 2}(\mu_{v}, \mu_{w})\operatorname{Var}(\bar w)+O( |\cD|^{-2} )
\end{aligned}
\]
Note that $f_{v}^{\prime}(\mu_{v}, \mu_{w})=1/\mu_{w}$ and $f_{w}^{\prime}(\mu_{v}, \mu_{w})=-\mu_{v}/\mu_{w}^2$. Then we have
\[
\begin{aligned}
\cV_1 &= \frac{1}{\left(\mu_{w}\right)^{2}} \operatorname{Var}(\bar v)+2 \frac{-\mu_{v}}{\left(\mu_{w}\right)^{3}} \operatorname{Cov}(\bar v, \bar w)+\frac{\left(\mu_{v}\right)^{2}}{\left(\mu_{w}\right)^{4}} \operatorname{Var}(\bar w)+O( |\cD|^{-2} ) \\
&=\frac{\left(\mu_{v}\right)^{2}}{\left(\mu_{w}\right)^{2}}\left[\frac{\operatorname{Var}(\bar v)}{\left(\mu_{v}\right)^{2}}-2 \frac{\operatorname{Cov}(\bar v, \bar w)}{\mu_{v} \mu_{w}}+\frac{\operatorname{Var}(\bar w)}{\left(\mu_{w}\right)^{2}}\right]+O( |\cD|^{-2} ) \\
&=\frac{\bfE \left(\bar v-\frac{\mu_v}{\mu_w}\bar w\right)^2}{\mu^2_w}+O( |\cD|^{-2} )\\
&=\frac{\sum_{(u,i)} p_{u,i}(1-p_{u,i})h^2_{u,i}/\hat p^2_{u,i}}{\left(\sum_{(u,i)}p_{u,i}/\hat p_{u,i}\right)^2}+O( |\cD|^{-2} ),
\end{aligned}
\]
where $h_{u,i}=(e_{u,i}-\hat e_{u,i})-\sum_{(u,i)}\{p_{u,i}(e_{u,i}-\hat e_{u,i})/\hat p_{u,i}\}/\sum_{(u,i)}\{p_{u,i}/\hat p_{u,i}\}$ is a bounded difference between $e_{u,i}-\hat e_{u,i}$ and its weighted average. The conclusion that the SDR variance is bounded for any propensities is given directly by the self-normalized form of SDR, i.e., the bounded range of SDR is $[\delta_{\operatorname{min}}, \delta_{\operatorname{max}}]$.

\end{proof}



\subsection{Proof of Theorem 4} \label{app-b}
\begin{proof}[Proof of Theorem 4]
The McDiarmid's inequality states that for independent bounded random variables $X_{1}, X_{2}, \ldots X_{n}$, where $X_{i} \in \mathcal{X}_{i}$ for all $i$ and a mapping $f: \mathcal{X}_{1} \times \mathcal{X}_{2} \times \cdots \times \mathcal{X}_{n} \rightarrow \mathbb{R}$. Assume there exist constant $c_{1}, c_{2}, \ldots, c_{n}$ such that for all $i$,
$$
\sup _{x_{1}, \cdots, x_{i-1}, x_{i}, x_{i}^{\prime}, x_{i+1}, \cdots, x_{n}}\left|f\left(x_{1}, \ldots, x_{i-1}, x_{i}, x_{i+1}, \cdots, x_{n}\right)-f\left(x_{1}, \ldots, x_{i-1}, x_{i}^{\prime}, x_{i+1}, \cdots, x_{n}\right)\right| \leq c_{i}.
$$
Then, for any $\epsilon>0$,
$$
\P \left(\left|f\left(X_{1}, X_{2}, \cdots, X_{n}\right)-\mathbb{E}\left[f\left(X_{1}, X_{2}, \cdots, X_{n}\right)\right]\right| \geq \epsilon\right) \leq 2 \exp \left(-\frac{2 \epsilon^{2}}{\sum_{i=1}^{n} c_{i}^{2}}\right).
$$
In fact, equation (\ref{eq-s1}) implies that the SDR estimator can be written as  
\begin{align*}
\cE_{SDR}= {
\sum_{(u,i)\in  \cD}  
\frac{o_{u,i}(e_{u,i}-\hat e_{u,i})}{\hat p_{u,i}}
} \Big / {\sum_{(u,i)\in  \cD} 
\frac{o_{u,i}}{\hat p_{u,i}}
} + \hat \cE, 
\end{align*}
denoted as $f(o_{1,1}, \dots, o_{u,i}, \dots,  o_{U,I})$.  
{Note that} 
\begin{align}
&\sup _{o_{u,i}, o^{\prime}_{u,i}}\left|f\left(o_{1,1}, \dots, o_{u,i}, \dots,  o_{U,I}\right)-f\left(o_{1,1}, \dots, o^{\prime}_{u,i} \dots,  o_{U,I}\right)\right|\notag\\
&\leq 
\begin{cases}
\delta_{\operatorname{max}}-\dfrac{\delta_{u,i}/\hat p_{u,i}+\sum_{\cD\setminus(u,i)}o_{u,i}/\hat p_{u,i}\delta_{\operatorname{max}}}{1/\hat p_{u,i}+\sum_{\cD\setminus(u,i)}o_{u,i}/\hat p_{u,i}}, \quad &\text{if}\quad \delta_{u, i}\leq (\delta_{\operatorname{min}}+\delta_{\operatorname{max}})/2,\\
\dfrac{\sum_{\cD\setminus(u,i)}o_{u,i}/\hat p_{u,i}\delta_{\operatorname{min}}+\delta_{u,i}/\hat p_{u,i}}{\sum_{\cD\setminus(u,i)}o_{u,i}/\hat p_{u,i}+1/\hat p_{u,i}}-\delta_{\operatorname{min}}, \quad &\text{if}\quad \delta_{u, i}> (\delta_{\operatorname{min}}+\delta_{\operatorname{max}})/2, 
\end{cases}\label{thm4-1}
\end{align}
{where $\cD\setminus(u,i)$ is the set of $\cD$ excluding the element $(u,i)$.} 

{Next, we focus on analyzing the $\sum_{\cD\setminus(u,i)}o_{u,i}/\hat p_{u,i}$.}
The  Hoeffding's inequality states that for independent bounded random variables $X_{1}, \ldots, X_{n}$ that take values in intervals of sizes $\rho_{1}, \ldots, \rho_{n}$ with probability 1 and for any $\epsilon>0$, 
$$\P \Big ( \Big | \sum_{k} X_{k}-\bfE (\sum_{k} X_{k}) \Big | \geq \epsilon \Big ) \leq 2 \exp \left(\frac{-2 \epsilon^{2}}{\sum_{k} \rho_{k}^{2}}\right).$$
For $\sum_{\cD\setminus (u,i)}o_{u,i}/\hat p_{u,i}$, we have 
\begin{align*}
    \P \Big ( \Big | \sum_{\cD\setminus (u,i)}o_{u,i}/\hat p_{u,i}-\sum_{\cD\setminus(u,i)}p_{u,i}/\hat p_{u,i}\Big |\geq \epsilon)\leq 2 \exp \left( \frac{-2 \epsilon^{2} }{\sum_{\cD\setminus(u,i)} 1/\hat p^2_{u,i}}\right), 
\end{align*}
Setting the last term equals to $\eta/2$, and solving for $
\epsilon$ gives that 
with probability at least $1-\eta/2$, the following inequality holds 
\begin{equation}\label{thm4-2}
\Biggl | \sum_{\cD\setminus(u,i)} o_{u,i}/\hat p_{u,i}-\sum_{\cD\setminus (u,i)} p_{u,i}/\hat p_{u,i}\Biggr | \leq \sqrt{ \frac{1}{2} \log \frac{4}{\eta} {\sum_{\cD\setminus(u,i)}\frac{1}{\hat p^2_{u,i}}}}\triangleq \epsilon^{\prime}.
\end{equation}
Therefore, combining (\ref{thm4-1}) and (\ref{thm4-2}) yields that with probability at least $1-\eta/2$, 
$$
\begin{aligned}
&\sup _{o_{1,1}, \dots, o_{u,i}, o^{\prime}_{u,i}, \dots,  o_{U,I}}\left|f\left(o_{1,1}, \dots, o_{u,i}, \dots,  o_{U,I}\right)-f\left(o_{1,1}, \dots, o^{\prime}_{u,i} \dots,  o_{U,I}\right)\right|\\
&\leq 
{ \begin{cases}
\delta_{\operatorname{max}}-\dfrac{\delta_{u,i}/\hat p_{u,i}+(\sum_{\cD\setminus(u,i)}p_{u,i}/\hat p_{u,i}-\epsilon^{\prime})\delta_{\operatorname{max}}}{1/\hat p_{u,i}+(\sum_{\cD\setminus(u,i)}p_{u,i}/\hat p_{u,i}-\epsilon^{\prime})}, \quad &\text{if}\quad \delta_{u, i}\leq (\delta_{\operatorname{min}}+\delta_{\operatorname{max}})/2,\\
\dfrac{(\sum_{\cD\setminus(u,i)}p_{u,i}/\hat p_{u,i}-\epsilon^{\prime})\delta_{\operatorname{min}}+\delta_{u,i}/\hat p_{u,i}}{(\sum_{\cD\setminus(u,i)}p_{u,i}/\hat p_{u,i}-\epsilon^{\prime})+1/\hat p_{u,i}}-\delta_{\operatorname{min}}, \quad &\text{if}\quad \delta_{u, i}> (\delta_{\operatorname{min}}+\delta_{\operatorname{max}})/2,
\end{cases} }\\
&\leq 
\begin{cases}
(\delta_{\operatorname{max}}-\delta_{u, i})/\{1+\hat p_{u,i}(\sum_{\cD\setminus(u,i)}p_{u,i}/\hat p_{u,i}-\epsilon^{\prime})\}, \quad &\text{if}\quad \delta_{u, i}\leq (\delta_{\operatorname{min}}+\delta_{\operatorname{max}})/2,\\
(\delta_{u, i}-\delta_{\operatorname{min}})/\{1+\hat p_{u,i}(\sum_{\cD\setminus(u,i)}p_{u,i}/\hat p_{u,i}-\epsilon^{\prime})\}, \quad &\text{if}\quad \delta_{u, i}> (\delta_{\operatorname{min}}+\delta_{\operatorname{max}})/2, 
\end{cases}
\end{aligned}
$$
where $\delta_{u, i}=e_{u, i}-\hat{e}_{u, i}$ is the error deviation, $\delta_{\operatorname{min}}=\operatorname{min}_{(u,i)\in \cD}\delta_{u, i}$, and $\delta_{\operatorname{max}}=\operatorname{max}_{(u,i)\in \cD}\delta_{u, i}$.

Invoking  McDiarmid’s inequality leads to that 
\begin{footnotesize}
\begin{align*}
&\P\left(\left|\cE_{SDR}-\bfE_{{\cO}} (\cE_{SDR})\right|\geq \epsilon\right)\\
\leq{}& 2 \exp \Biggl \{ {-2 \epsilon^{2} } \Biggl/   \Biggl ( \sum_{(u,i): \delta_{u, i}\leq \frac{\delta_{\operatorname{min}}+\delta_{\operatorname{max}}}{2}} \frac{(\delta_{\operatorname{max}}-\delta_{u, i})^2}{\{1+\hat p_{u,i}(\sum_{\cD-(u,i)}p_{u,i}/\hat p_{u,i}-\epsilon^{\prime})\}^2} \\
&{}+ \sum_{(u,i): \delta_{u, i}> \frac{\delta_{\operatorname{min}}+\delta_{\operatorname{max}}}{2}}\frac{(\delta_{u, i}-\delta_{\operatorname{min}})^2}{\{1+\hat p_{u,i}(\sum_{\cD-(u,i)}p_{u,i}/\hat p_{u,i}-\epsilon^{\prime})\}^2} \Biggr  ) \Biggr \} \\
\leq{}& 2 \exp \left(\frac{-2 \epsilon^{2}}{\sum_{(u,i)} \{(\delta_{\operatorname{max}}-\delta_{u, i})^2+(\delta_{u, i}-\delta_{\operatorname{min}})^2\}/\{1+\hat p_{u,i}(\sum_{\cD-(u,i)}p_{u,i}/\hat p_{u,i}-\epsilon^{\prime})\}^2}\right)
\end{align*}
\end{footnotesize}
Setting the last term equals to $\eta/2$, and solving for $
\epsilon$ complete the proof.

\end{proof}

\subsection{Generalization Bound under Inaccurate Models}   \label{app-e}
\begin{theorem}[Generalization Bound under Inaccurate Models] \label{th5} 
For any finite hypothesis space of predictions $\mathcal{H}=\{\hat{\mathbf{Y}}^{1}, \ldots, \hat{\mathbf{Y}}^{|\mathcal{H}|}\}$, with probability $1-\eta$, the true risk $R(\hat{\mathbf{Y}}^{\dagger})$ deviates from the SDR estimator with imputed errors $\hat e_{u,i}$ and learned propensities $\hat p_{u,i}$ satisfying the stabilization constraint \ref{step2} is bounded by
$$
\begin{aligned}
R(\hat{\mathbf{Y}}^{\dagger}) \leq{}&  \mathcal{\hat E}_{SDR}(\hat{\mathbf{Y}}^{\dagger})+ \underbrace{\Biggl |\frac{1}{|\cD|} \sum_{(u,i)\in  \cD}\delta^\dagger_{u,i}-\frac{\sum_{(u,i)\in  \cD}{\delta^\dagger_{u,i}p_{u,i}}/{\hat p_{u,i}}}{\sum_{(u,i)\in  \cD}{p_{u,i}}/{\hat p_{u,i}}}\Biggr |}_{\text {Bias Term }}\\
&+\underbrace{\sqrt{\frac{1}{2}{\log \left(\frac{4|\mathcal{H}|}{\eta}\right)} \sum_{ (u, i) \in \mathcal{D}}\frac{(\delta_{\operatorname{max}}-\delta^\S_{u, i})^2+(\delta^\S_{u, i}-\delta_{\operatorname{min}})^2}{\{1+\hat p_{u,i}(\sum_{\cD\setminus(u,i)}p_{u,i}/\hat p_{u,i}-\epsilon^{\prime})\}^2}}}_{\text {Variance Term }}
\end{aligned}
$$
where $\delta_{u, i}^{\S}$ is the error deviation corresponding to the prediction model $$\hat{\mathbf{Y}}^{\S}=\operatorname{argmax}_{ \hat{\mathbf{Y}}^{h} \in \mathcal{H}}{\sum_{ (u, i) \in \mathcal{D}}\frac{(\delta_{\operatorname{max}}-\delta^\S_{u, i})^2+(\delta^\S_{u, i}-\delta_{\operatorname{min}})^2}{\{1+\hat p_{u,i}(\sum_{\cD\setminus(u,i)}p_{u,i}/\hat p_{u,i}-\epsilon^{\prime})\}^2}}.$$
\end{theorem}
\begin{proof}[Proof of Theorem 5]
Proof. Theorem \ref{th4} shows that for all predictions $\hat{\mathbf{Y}}^h\in \mathcal{H}$, we have
\begin{align*}
&P\left(\left|\cE_{SDR}(\hat{\mathbf{Y}}^h)-\bfE[\cE_{SDR}(\hat{\mathbf{Y}}^h)]\right|\geq \epsilon\right)\\
\leq{}& 2 \exp \left(\frac{-2 \epsilon^{2}}{\sum_{(u,i)} \{(\delta_{\operatorname{max}}-\delta^h_{u, i})^2+(\delta^h_{u, i}-\delta_{\operatorname{min}})^2\}/\{1+\hat p_{u,i}\}^2}\right)
\end{align*}
McDiarmid's inequality and union bound ensures the following uniform convergence results:
\begin{align*}
&P\left(\left|\cE_{SDR}(\hat{\mathbf{Y}}^\dagger)-\bfE[\cE_{SDR}(\hat{\mathbf{Y}}^\dagger)]\right| \leq \epsilon\right) \geq 1-\eta \\
&\Leftarrow P\left(\max _{\hat{\mathbf{Y}}^h \in \mathcal{H}}\left|\cE_{SDR}(\hat{\mathbf{Y}}^h)-\bfE[\cE_{SDR}(\hat{\mathbf{Y}}^h)]\right| \leq \epsilon\right) \geq 1-\eta \\
&\Leftrightarrow P\left(\bigvee_{\hat{\mathbf{Y}}_{i} \in \mathcal{H}}\left|\cE_{SDR}(\hat{\mathbf{Y}}^h)-\bfE[\cE_{SDR}(\hat{\mathbf{Y}}^h)]\right| \geq \epsilon\right)<\eta \\
&\Leftarrow \sum_{h=1}^{|\mathcal{H}|} P\left(\left|\cE_{SDR}(\hat{\mathbf{Y}}^h)-\bfE[\cE_{SDR}(\hat{\mathbf{Y}}^h)]\right| \geq \epsilon\right)<\eta \\
&\Leftarrow \sum_{h=1}^{|\mathcal{H}|} 2 \exp \left(\frac{-2 \epsilon^{2}}{\sum_{(u,i)} \{(\delta_{\operatorname{max}}-\delta^h_{u, i})^2+(\delta^h_{u, i}-\delta_{\operatorname{min}})^2\}/\{1+\hat p_{u,i}(\sum_{\cD-(u,i)}p_{u,i}/\hat p_{u,i}-\epsilon^{\prime})\}^2}\right)<\eta \\
&\Leftarrow|\mathcal{H}| \cdot 2 \exp \left(\frac{-2 \epsilon^{2}}{\sum_{(u,i)} \{(\delta_{\operatorname{max}}-\delta^\S_{u, i})^2+(\delta^\S_{u, i}-\delta_{\operatorname{min}})^2\}/\{1+\hat p_{u,i}(\sum_{\cD-(u,i)}p_{u,i}/\hat p_{u,i}-\epsilon^{\prime})\}^2}\right)<\eta
\end{align*}
Solving the last inequality for $\epsilon$, it is concluded that, with probability $1-\eta$, the following inequality holds
$$ \bfE[\cE_{SDR}(\hat{\mathbf{Y}}^\dagger)]-\cE_{SDR}(\hat{\mathbf{Y}}^\dagger) \leq \sqrt{\frac{1}{2}{\log \left(\frac{4|\mathcal{H}|}{\eta}\right)} \sum_{ (u, i) \in \mathcal{D}}\frac{(\delta_{\operatorname{max}}-\delta^\S_{u, i})^2+(\delta^\S_{u, i}-\delta_{\operatorname{min}})^2}{\{1+\hat p_{u,i}(\sum_{\cD\setminus(u,i)}p_{u,i}/\hat p_{u,i}-\epsilon^{\prime})\}^2}}.
$$ 
Theorem \ref{thbias} shows that for the optimal prediction model $\hat{\mathbf{Y}}^\dagger$, the following inequality holds
$$
\begin{aligned}
R(\hat{\mathbf{Y}}^{\dagger})-\bfE[\cE_{SDR}(\hat{\mathbf{Y}}^\dagger)] \leq \Biggl |\frac{1}{|\cD|} \sum_{(u,i)\in  \cD}\delta^\dagger_{u,i}-\frac{\sum_{(u,i)\in  \cD}{\delta^\dagger_{u,i}p_{u,i}}/{\hat p_{u,i}}}{\sum_{(u,i)\in  \cD}{p_{u,i}}/{\hat p_{u,i}}}\Biggr |
\end{aligned}.
$$
The stated results can be obtained by adding the two inequalities above.

\end{proof}

\section{Further Theoretical Analysis of SDR} \label{app-f}

Without loss of generality,  we assume 
    $$ \frac{1}{|\mathcal{D}|} \sum_{(u,i)\in \mathcal{D} } \frac{o_{u,i}}{\hat p_{u,i}} ( \hat e_{u,i} - \mathcal{\hat E} )  = \lambda, \quad  \lambda \neq 0.$$
In this case, the learned propensities must be inaccurate; otherwise, the constraint (\ref{step2}) holds naturally as the same size increases. 
Thus, if the imputed errors are accurate, then $\mathcal{L}_{ideal} =  \mathcal{\hat E}$.  
    By a exactly same arguments of equation (\ref{eq5}), we have 
        $$ \frac{1}{|\mathcal{D}|} \sum_{(u,i)\in \mathcal{D} } \frac{o_{u,i}}{\hat p_{u,i}} ( \mathcal{\hat E} - \mathcal{E}_{SDR} )  = \lambda, $$
which implies that 
    $$  \mathcal{E}_{SDR}   = \mathcal{L}_{ideal} -  \lambda \Big / \frac{1}{|\mathcal{D}|} \sum_{(u,i)\in \mathcal{D} } \frac{o_{u,i}}{\hat p_{u,i}}.$$
        This means that the degree of violation of constraint (\ref{step2}) determines the size of the bias of SDR.

Furthermore, we can compute the bias, variance, tail bound, and generalization error bound of SDR. Specifically, if both the learned propensities and imputed errors are inaccurate, constraint (3) does not hold either. Then the bias of SDR is   
    $$ 
{\text{Bias}}(\mathcal{E}_{SDR}) = \Biggl |\frac{1}{|\mathcal{D}|} \sum_{(u,i)\in  \mathcal{D}} \left ( e_{u, i}-\frac{\sum_{(u,i)\in  \mathcal{D}}{ e_{u, i}p_{u,i}}/\hat p_{u,i}}{\sum_{(u,i)\in  \mathcal{D}}{p_{u,i}}/\hat p_{u,i}} \right ) \Biggr | + O( |\mathcal{D}|^{-1} ),
 $$
the variance of SDR becomes 
    $$
\text{Var}\left(\mathcal{E}_{SDR}\right)=\frac{\sum_{(u,i)} p_{u,i}(1-p_{u,i})\tilde h^2_{u,i}/\hat p^2_{u,i}}{\left(\sum_{(u,i)}p_{u,i}/\hat p_{u,i}\right)^2}+O( |\mathcal{D}|^{-2} ),
 $$
where $\tilde h_{u,i} = e_{u,i} - \sum_{(u,i)\in \mathcal{D}} \{ p_{u,i} e_{u,i} / \hat p_{u,i} \} \Big /  \sum_{(u,i)\in \mathcal{D}} \{ p_{u,i} / \hat p_{u,i} \}.$ 
 The tail bound of SDR is given as 
    $$\left|\mathcal{E}_{SDR}-\mathbb{E}_{{\mathcal{O}}} (\mathcal{E}_{SDR})\right| \leq \sqrt{\frac{1}{2}{\log \left(\frac{4}{\eta}\right)} \sum_{ (u, i) \in \mathcal{D}}\frac{(e_{\text{max}}-e_{u, i})^2+(e_{u, i}-e_{\text{min}})^2}{\{1+\hat p_{u,i}(\sum_{\mathcal{D}\setminus(u,i)}p_{u,i}/\hat p_{u,i}- \epsilon^{\prime} )\}^2}},$$
where $\delta_{\text{min}}=\text{min}_{(u,i)\in \mathcal{D}}e_{u, i}$, $\delta_{\text{max}}=\text{max}_{(u,i)\in \mathcal{D}}e_{u, i}$,  $\epsilon^{\prime}={\small \sqrt{ \log(4/\eta) / 2 \cdot   \sum_{\mathcal{D}\setminus(u,i)} 1/ \hat p^2_{u,i}}}$, and $\mathcal{D}\setminus(u,i)$ is the set of $\mathcal{D}$ excluding the element $(u,i)$.  In addition, we can derive the generation error bound of SDR.  
Given a finite hypothesis space $\mathcal{H}$ of the prediction model, then for any a prediction model $h \in \mathcal{H}$,  with probability $1-\eta$, the true risk $R(h)$ deviates from the SDR estimator is bounded by 
$$R(h) \leq  \mathcal{\hat E}_{SDR}(h) + 
{\text{Bias}}(\mathcal{E}_{SDR}) + \sqrt{\frac{1}{2}{\log (\frac{4|\mathcal{H}|}{\eta})} \sum_{ (u, i) \in D}\frac{(e_{\text{max}}-e^\S_{u, i})^2+(e^\S_{u, i}-e_{\text{min}})^2}{\{1+\hat p_{u,i}(\sum_{\mathcal{D}\setminus(u,i)}p_{u,i}/\hat p_{u,i}-\epsilon^{\prime})\}^2}},$$
where $e_{u, i}^{\S}$ is the error deviation corresponding to the prediction model
$$h^{\S}=\arg\max_{ h \in \mathcal{H}} \sum_{ (u, i) \in D}\frac{(e_{max}-e^\S_{u, i})^2+(e^\S_{u, i}-e_{min})^2}{\{1+\hat p_{u,i}(\sum_{D-(u,i)}p_{u,i}/\hat p_{u,i}-\epsilon^{\prime})\}^2}.$$

\end{document}